\documentclass[twoside]{article}
\usepackage{mor}

\usepackage{amsmath}
\usepackage[normalem]{ulem}
\usepackage{cases}

\received{July 17, 2010}                 %
\revised{September 7, 2010 and September 16, 2010} 
\pubyear{0x}                             %
\pubmonth{Xxxxxxx}                       %
\volume{xx}                              %
\issue{x}                                %
\pages{xxx--xxx}                         %
\firstpage{0xxx}                         %
\DOI{10.1287/moor.xxxx.xxxx}             %
\startpagenumber{1}                      %
\title{A Geometric Proof of Calibration}
\ShortTitle{A Geometric Proof of Calibration}
\ShortAuthors{Mannor and Stoltz}
\NumberOfAuthors{2}
\FirstAuthor{Shie Mannor}
\FirstAuthorAddress{Israel Institute of Technology (Technion), Haifa, Israel}
\FirstAuthorEmail{shie@ee.technion.ac.il}
\FirstAuthorURL{http://webee.technion.ac.il/people/shie/}
\SecondAuthor{Gilles Stoltz}
\SecondAuthorAddress{Ecole Normale Sup{\'e}rieure -- CNRS -- INRIA, Paris, France \& HEC Paris -- CNRS, Jouy-en-Josas, France}
\SecondAuthorEmail{gilles.stoltz@ens.fr}
\SecondAuthorURL{http://www.math.ens.fr/~stoltz}

\renewcommand{\epsilon}{\varepsilon}
\newcommand{\e}{\varepsilon}
\newcommand{\cA}{\mathcal{A}}
\newcommand{\cT}{\mathcal{T}}
\newcommand{\cB}{\mathcal{B}}
\newcommand{\cP}{\mathcal{P}}
\newcommand{\norm}[1][\,\cdot\,]{\ensuremath{\left\Arrowvert #1 \right\Arrowvert}}
\newcommand{\bp}{\mathbf{p}}
\newcommand{\bq}{\mathbf{q}}
\newcommand{\ind}{\mathbb{I}}
\newcommand{\uz}{\underline{0}}
\newcommand{\R}{\mathbb{R}}
\renewcommand{\leq}{\leqslant}
\renewcommand{\geq}{\geqslant}


\MSCcodes{Primary: 91A05, 91A20, 91A26; Secondary: 68Q32} 

\ORMScodes{Primary: Games/group decisions--noncooperative}

\begin{document}
\maketitle

\begin{abstract}
We provide yet another proof of the existence of calibrated forecasters; it has two merits.
First, it is valid for an arbitrary finite number of outcomes.
Second, it is short and simple and it follows from a direct application
of Blackwell's approachability theorem to carefully chosen vector-valued payoff function
and convex target set. Our proof captures the essence
of existing proofs based on approachability (e.g., the proof by Foster~\cite{Fo99} in case of binary outcomes)
and highlights the intrinsic connection between approachability and calibration. \\

Received: July 17, 2010; revised: September 7, 2010; final version: September 16, 2010
\end{abstract}

\section{Motivation.}

Foster~\cite{Fo99} stated that:
\begin{quote}
{\small
Over the past few years many proofs of the existence of calibration have been discovered. Each of the following
provides a different algorithm and proof of convergence: Foster and Vohra~\cite{FoVo91,FoVo98}, Hart~\cite{Ha95},
Fudenberg and Levine~\cite{FuLe99}, Hart and Mas-Colell~\cite{HaMa97}. Does the literature really need one more? Probably not.
}
\end{quote}
In spite of this, he argued, successfully, that his new proof of the existence of calibrated forecasters
in the case of binary outcomes, based on Blackwell's approachability theorem (Blackwell~\cite{Bla56}),
was shorter and more direct than most of the previous proofs.

In this paper, we consider the general case of finitely many outcomes and exhibit an even shorter (ten-line long)
proof of the existence of calibrated forecasters based on approachability. We show
therefore that calibration is a straightforward
consequence of approachability. As we realized by browsing on the web,
approachability and calibration are well-taught matters and we are confident that this new proof
will become a standard example in the list of direct applications of approachability (as is already the
case for the existence of no-regret forecasters).
Since calibration is a central tool in learning in games (see, e.g., Kakade and Foster~\cite{KakadeFoster08}) and in online learning
(see, e.g., Mannor, Tsitsiklis, and Yu~\cite{MannorTY09}), 
the simplicity of the proof and the guaranteed convergence rates open up new opportunities to
use calibration in practical learning algorithms.

Foster~\cite{Fo99} mentions that his approachability-based proof of the
existence of a calibrated forecaster was obtained by first considering a modification
of an intuitive forecaster already stated in Foster and Vohra~\cite{FoVo91} and then working out
the proof of its guarantees. We proceed the other way round and start directly from the statement of
Blackwell's approachability theorem for convex sets \cite[Theorem~3]{Bla56}
but, as a drawback, can only exhibit a forecaster which has to solve
a linear program at each step. Taking a closer look at Foster~\cite{Fo99},
one can see that we indeed capture the essence of his previous proof.
His algorithm is a clever modification, in the case of binary outcomes,
of the general approachability-based forecaster presented below; the former
has a nice, explicit, and simple statement.
\bigskip

We now recall the informal definition and consequences of calibration.
Consider a finite set of possible outcomes and suppose we obtain random forecasts
about future events; these forecasts are each given by probability
distributions over the outcomes. Now, such a sequence of forecasts
is called calibrated whenever it is consistent in hindsight, that is,
when, for all distributions $\bp$, the actual empirical distribution of the outcomes
on those rounds when the forecast was close to $\bp$ is also close to $\bp$.

Having a calibrated forecasting scheme is beneficial in several ways. On the one hand,
it allows some agent to choose the best responses to the predicted forecasts
or to consider other risk measures which might be more valuable than greedily choosing
the best action leading to highest reward. On the other hand,
calibrated forecasting rules enable multiple agents to converge to a reasonable joint play in some
situations. For instance, if all players use calibrated forecasts of other players' actions,
then the empirical distribution of action profiles converges to the set of correlated equilibria;
see Foster and Vohra~\cite{FoVo97}.
We refer to Sandroni, Smorodinsky, and Vohra~\cite{SaSmVo03} for further discussion on calibrated forecasting as well as its generalizations.

\section{Setup and formal definition of calibration.}
\label{sec:caldef}

We consider a finite set $\cA$ of outcomes, with cardinality denoted by $A$ and
denote by $\cP = \Delta(\cA)$ the set of probability distributions over $\cA$.
We equip $\cP$, which can be considered a subset of $\R^{A}$,
with some\footnote{The precise nature of this norm, e.g., $\ell^1$, Euclidian $\ell^2$,
or $\ell^{\infty}$ supremum norm, is irrelevant at this stage, since all norms
are equivalent on finite-dimensional spaces.} norm $\norm$, to be referred to as
the calibration norm.
In particular, the Dirac probability distribution on some outcome $a \in \cA$
will be referred to as $\delta_a$.

A forecaster plays a game against Nature. At each step, it outputs a probability
distribution $P_t \in \cP$ while Nature chooses simultaneously an outcome $a_t \in \cA$.
We make no assumption on Nature's strategy.

The goal of the forecaster is to ensure the following property, known as calibration:
for all strategies of Nature,
\begin{equation}
\label{def:cal}
\forall \, \e > 0, \ \ \forall \, \bp \in \cP, \qquad \quad
\lim_{T \to +\infty} \norm[\frac{1}{T} \sum_{t=1}^T \ind_{ \bigl\{ \norm[P_t - \bp] \leq \e \bigr\} }
\bigl( P_t - \delta_{a_t} \bigr)] = 0 \qquad \quad \mbox{a.s.}
\end{equation}
The {a.s.} statement accounts for randomized forecasters. (It was shown by Oakes~\cite{Oakes85} and Dawid~\cite{Dawid85} that
randomization is essential for calibration.) \\

The literature (e.g., Foster and Vohra~\cite{FoVo98}, 
Foster~\cite{Fo99}) essentially considers a less ambitious goal, at least in a first step: $\e$--calibration.
(We explain in Section~\ref{sec:cstr} how to get a calibrated forecaster from some
sequence of $\e$--calibrated forecasters with good properties.) Formally, given $\e > 0$, an $\e$--calibrated forecaster considers some
finite covering of $\cP$ by $N_{\e}$ balls of radius ${\e}$ and abides by the following constraints.
Denoting by $\bp_1,\ldots,\bp_{N_{\e}}$ the centers of the balls in the covering (they form what
will be referred to later on as an $\e$--grid),
the forecaster chooses only forecasts $P_t \in \bigl\{ \bp_1,\ldots,\bp_{N_{\e}} \bigr\}$.
We thus denote by $K_t$ the index in $\bigl\{ 1,\ldots,N_{\e} \bigr\}$ such that $P_t = \bp_{K_t}$.
The final condition to be satisfied is then that
for all strategies of Nature,
\begin{equation}
\label{def:epscal}
\limsup_{T \to +\infty} \ \ \sum_{k=1}^{N_\e} \norm[\frac{1}{T} \sum_{t=1}^T \ind_{ \{ K_t = k \} }
\bigl( \bp_k - \delta_{a_t} \bigr)] \,\, \leq \e \qquad \quad \mbox{a.s.}
\end{equation}

When the calibration norm is the $\ell^1$--norm $\norm_1$, the sum appearing in this criterion is usually
referred to as the $\ell^1$--calibration score (Foster~\cite{Fo99}). Another popular criterion
is the Brier score (Foster and Vohra~\cite{FoVo98}), which we consider in Section~\ref{sec:Brier};
it is bounded, up to a factor of 2, by the $\ell^1$--calibration score.

\section{A geometric construction of $\epsilon$--calibrated forecasters.}

In this section we prove our main result regarding the existence of an $\epsilon$--calibrated forecaster
based on approachability theory. We recall results approachability theory, provide the main result (Theorem~\ref{th:main}),
and then address the issue of computational complexity.

\subsection{Statement of Blackwell's approachability theorem.}

Consider a vector-valued game between two players,
with respective finite action sets $\mathcal{I}$ and $\mathcal{J}$.
We denote by $d$ the dimension of the reward vector.
The payoff function of the first player is given by
a mapping $m : \mathcal{I} \times \mathcal{J} \to \R^d$,
which is linearly extended to $\Delta(\mathcal{I}) \times \Delta(\mathcal{J})$,
the set of product-distributions over $\mathcal{I} \times \mathcal{J}$.

We denote by $I_1,I_2,\ldots$ and $J_1,J_2,\ldots$ the sequences of actions in $\mathcal{I}$
and $\mathcal{J}$ taken by each player (they are possibly given by randomized strategies).
Let $C \subset \R^d$ be some set. By definition, $C$ is approachable if there exists a strategy
for the first player such that for all strategies of the second player,
\[
\lim_{T \to \infty} \ \ \ \inf_{c \in C} \ \norm[c - \frac{1}{T} \sum_{t=1}^T m \bigl( I_t,J_t \bigr)] \ = 0
\qquad \quad \mbox{a.s.}
\]
That is, the first player has a strategy that ensures that the average of his vector-valued payoffs
converges to the set $C$.

For closed convex sets $C$, there is a simple characterization of approachability that is a
direct consequence of the minimax theorem.

\begin{theorem}[Blackwell {\cite[Theorem~3]{Bla56}}]
\label{th:appr}
A closed convex set $C \subset \R^d$ is approachable if and only if
\[
\forall \, \bq \in \Delta(\mathcal{J}), \ \ \exists \, \bp \in \Delta(\mathcal{I}), \qquad \quad
m(\bp,\bq) \in C~.
\]
\end{theorem}

\subsection{Application to the existence of an $\e$--calibrated forecaster.}

As indicated above, we equip $\cP$ with some calibration norm $\norm$ and
fix $\e > 0$; we then consider an associated $\e$--grid $\bigl\{ \bp_1,\ldots,\bp_{N_\e} \bigr\}$
in $\cP = \Delta(\cA)$.

\begin{theorem}
\label{th:main}
There exists an $\e$--calibrated forecaster which selects at every stage a distribution from this grid.
\end{theorem}

\begin{proof}
We apply the results on approachability recalled above.
To that end, we consider in our setting the action sets
$\mathcal{I} = \{ 1, \ldots, N_\e \}$ for the first player
and $\mathcal{J} = \cA$ for the second player.

We define the vector-valued payoff function as follows; it takes values in $\R^{A N_\e}$.
For all $k \in \{ 1,\ldots,N_\e \}$ and $a \in \cA$,
\[
m(k,a) = \bigl( \uz, \, \ldots, \, \uz, \,\, \bp_k - \delta_a, \,\, \uz, \, \ldots, \, \uz \bigr)~,
\]
which is a vector of $N_\e$ elements of $\R^A$ composed by $N_\e-1$
occurrences of the zero element $\uz \in \R^A$ and one non-zero element, located in the $k$--th position
and given by the difference of probability distributions $\bp_k - \delta_a$.

We now define the target set $C$ as the following subset of the
$\e$--ball around $\bigl( \uz, \, \ldots, \, \uz \bigr)$ for the calibration norm $\norm$.
We write $(A N_\e)$--dimensional vectors of $\R^{A N_\e}$ as $N_\e$--dimensional
vectors with components in $\R^A$, i.e., for all $\uuline{x} \in \R^{A N_\e}$,
\[
\uuline{x} = \bigl( \underline{x}_1, \, \ldots, \underline{x}_{N_\e} \bigr)~,
\]
where $\underline{x}_k \in \R^A$ for all $k \in \{ 1,\ldots,N_\e \}$.
Then,
\[
C = \left\{ \uuline{x} \in \R^{A N_\e} : \ \sum_{k=1}^{N_\e} \norm[\underline{x}_k] \,\, \leq \e \right\}~.
\]
Note that $C$ is a closed convex set.

The condition (\ref{def:epscal}) of $\e$--calibration can be rewritten as follows: the sequence of the vector-valued rewards
\[
\overline{m}_T \stackrel{\mbox{\scriptsize def}}{=} \frac{1}{T} \sum_{t=1}^T m \bigl( K_t,a_t \bigr)
= \left( \frac{1}{T} \sum_{t=1}^T \ind_{ \{ K_t = 1 \} }
\bigl( \bp_1 - \delta_{a_t} \bigr), \,\, \ldots, \,\,
\frac{1}{T} \sum_{t=1}^T \ind_{ \{ K_t = N_\e \} }
\bigl( \bp_{N_\e} - \delta_{a_t} \bigr) \right)
\]
converges to the set $C$ almost surely.

The existence of an $\e$--calibrated forecaster is thus equivalent to the
approachability of $C$, which we now prove by showing that the
characterization provided by Theorem~\ref{th:appr} is satisfied.
Let $\bq \in \Delta(\mathcal{J}) = \cP$.
By construction, there exists $k \in \{ 1,\ldots,N_\e \}$ such that
$\norm[\bp_k - \bq] \leq \e$ and thus
\[
m(k,\bq) \in C~.
\]
(Here, the distribution $\bp$ of the approachability theorem can be taken as the
Dirac distribution $\delta_k$.)
\end{proof}

\subsection{Computation of the exhibited $\e$--calibrated forecaster.}
\label{sec:cplx}

The proof of the approachability theorem gives rise to an implicit
strategy, as indicated in Blackwell~\cite{Bla56}.
We denote here by $\Pi_C$ the projection in $\ell^2$--norm onto $C$.

At each round $t \geq 2$ and with the notations above, the forecaster should
pick his action $K_t$ at random according to a distribution $\psi_t = \bigl( \psi_{t,1}, \ldots,
\psi_{t,N_\e} \bigr)$ on $\bigl\{ 1,\ldots,N_\e \bigr\}$ such that
\begin{equation}
\label{eq:Blk}
\forall \, a \in \cA, \quad \qquad
\Bigl( \overline{m}_{t-1} - \Pi_C \bigl( \overline{m}_{t-1} \bigr) \Bigr) \,\cdot\,
\Bigl( m \bigl( \psi_t, \, a \bigr) - \Pi_C \bigl( \overline{m}_{t-1} \bigr) \Bigr) \leq 0~,
\end{equation}
where $\,\cdot\,$ denotes the inner product in $\R^{A N_\e}$.
The proof of Theorem~\ref{th:appr} (see Blackwell~\cite{Bla56}) shows that such a distribution $\psi_t$ indeed exists;
the question is how to efficiently compute it.
To do so, we first need to compute the projection
$\Pi_C \bigl( \overline{m}_{t-1} \bigr)$ of $\overline{m}_{t-1}$.

We address the two computational issues separately.
We first indicate how to find the projection efficiently and then
explain how to find the distribution $\psi_t$ based on the knowledge of this projection.

\subsubsection{Projecting onto $C$.}

We need to find the closest point in $C$ to $\overline{m}_{t-1}$.
Since $C$ is convex and the $\ell^2$--norm is convex, we have to deal with
a minimization problem of a convex function over a convex set.
Since answering the question whether a given point is in $C$ or not can be done in time linear in $A N_\e$,
the projection problem can be solved (approximately) in time polynomial in $A N_\e$.

\bigskip
Now, for the special case where the calibration norm is the $\ell^1$--norm $\norm_1$,
we can do much better.
For $i \in \bigl\{ 1, \ldots, A N_{\e} \bigr\}$, we denote by $s_{i,t-1} \in \{ -1,1\}$
the sign of the $i$--th component $\overline{m}_{i,t-1}$
of the vector $\overline{m}_{t-1}$. (The value of the sign function
at $x$ is arbitrary at $x=0$, equal to $-1$ when $x <0$ and to $1$ when $x >0$.) Then,
$\Pi_C \bigl( \overline{m}_{t-1} \bigr)$ is the solution of the following optimization problem,
where the unknown is $\uuline{y} = \bigl( y_1,\ldots,y_{A N_\e} \bigr)$:
\begin{align}
\nonumber
\min \quad & \bigl\| \uuline{y} - \overline{m}_{t-1} \bigr\|_2^2 \vspace{.2cm} \\
\nonumber
\mbox{such that} \ \ & \left\{
\begin{array}{cc}
\displaystyle{\sum_{i=1}^{A N_\e} y_i \, s_{i,t-1} \leq \e} & \vspace{.15cm}\\
  \quad y_i \, s_{i,t-1} \geq 0~, & \forall \, i \in \bigl\{ 1, \ldots, A N_{\e} \bigr\}~.
\end{array}
\right.
\end{align}
It can be easily shown (as in Gafni and Bertsekas~\cite{GafniBertsekas84} or by an immediate adaptation of
Palomar~\cite[Lemma~1]{Palomar05})
that the optimal solution is unique; it is given by $\uuline{y}(\mu^*)$ where
for all $\mu \geq 0$,
\[
\uuline{y}(\mu) = s_{i,t-1} \, \bigl( s_{i,t-1} \, \overline{m}_{i,t-1} - \mu \bigr)^+
\]
and $\mu^*$ is chosen as the minimum nonnegative value such that
$\sum_i y_i(\mu) \, s_{i,t-1} \leq \e$.
(Note that if $\mu^* > 0$ then $\sum_i y_i(\mu^*) \, s_{i,t-1} = \e$.) Finding $\mu^*$ can be done
by a binary search to an arbitrary precision.

In conclusion, when the calibration norm is the $\ell^1$--norm $\norm_1$, projecting onto $C$ can be done in linear time in $A N_{\e}$ to a desired
precision $\delta$ with complexity that depends on $\delta$ like $\log(1/\delta)$.

\subsection{Finding the optimal distribution $\psi_t$ in (\ref{eq:Blk}).}

The question that has to be resolved is therefore how to find $\psi_t$ that satisfies condition~(\ref{eq:Blk}).
Since we know that such a $\psi_t$ exists,
it suffices, for instance, to compute an element of
\[
\mathop{\mathrm{argmin}}_{\psi} \,\, \max_{a \in \cA} \,\,
\Bigl( \overline{m}_{t-1} - \Pi_C \bigl( \overline{m}_{t-1} \bigr) \Bigr) \,\cdot\, m ( \psi, \, a )
= \mathop{\mathrm{argmin}}_{\psi} \,\, \max_{a \in \cA} \,\,
\sum_{k=1}^{N_\e} \, \psi_k \, \gamma_{k,a,t-1}
\]
where we denoted $\gamma_{k,a,t-1} = \Bigl( \overline{m}_{t-1} - \Pi_C \bigl( \overline{m}_{t-1} \bigr) \Bigr)
\,\cdot\, m ( k, \, a )$.

This can be done efficiently by linear programming leading to a polynomial complexity in $N_\e$ and $A$.

\bigskip
However, if instead of solving the minimax problem exactly we are satisfied with solving it approximately,
i.e., allowing a small violation $\delta > 0$ in each of the $A$ constraints given by (\ref{eq:Blk}),
we can use the multiplicative weights algorithm as explained in Freund and Schapire~\cite{FrSc99}; see also
Cesa-Bianchi and Lugosi~\cite[Section~7.2]{CeLu06}.
The complexity of such a solution would be
\[
O \! \left(\frac{A N_\e}{\delta^2} \ln N_\e \right)~,
\]
since $(\ln N_{\e})/\delta^2$ steps of complexity $A N_{\e}$ each have to be performed.

The proof of Blackwell's approachability theorem shows that in this case
the sequence of the average payoff vectors $\overline{m}_{t}$
converges rather to the $\sqrt{\delta}$--expansion (in $\ell^2$--norm) of $C$;
it is easy\footnote{It suffices to note that for all vectors $\Delta$ of a finite-dimensional
space, one has $\| \Delta \|_\infty \leq \| \Delta \|_2$,
so that the inequality
$\| \Delta \|_2 \leq \sqrt{\| \Delta \|_\infty \, \| \Delta \|_1}$
yields
$\sqrt{\| \Delta \|_2} \leq \| \Delta \|_1$.
}
to see that the latter is included in the
$\delta$--expansion (in $\ell^1$--norm) of $C$.

Putting all things together and taking the $\ell^1$--norm $\norm_1$ as the calibration norm
(in particular, to define $C$), we can find a $2\e$--calibrated forecaster
whose complexity is of the order of $A N_\e \, \e^{-2} \log N_\e$ at each step.
Since $N_\e$ behaves like $\e^{-(A-1)}$ we have that the dependence of the
complexity per stage behaves like $\e^{-(A+1)}$ (ignoring multiplicative and logarithmic factors).
This implies a polynomial dependence in $\e$ but an exponential dependence in $A$.

\begin{remark} 
It is worth noting that when choosing a solution $\psi_t$, it is not possible to replace $\psi_t$ with
its mean or with an element of $\bp_1,\bp_2,\ldots, \bp_{N_\e}$ that is close to its mean.
The reason is that this would give rise to a deterministic rule, which, as we mentioned
in Section~\ref{sec:caldef}, cannot be calibrated. The fact that we have to randomize rather than take the mean
is due to our construction of the vector-valued game; therein,
playing a mixed action $\psi_t$ over the $\bp_i$'s leads to a very different
vector-valued reward than playing the (element $\bp_k$ closest to the) mean of the mixed action.
This is because different indices of the $(A N_\e)$--dimensional space are involved.
\end{remark}

\section{Rates of convergence and construction of a calibrated forecaster.}

In this section we provide rates of convergence and discuss the construction of a
calibrated (rather than $\e$--calibrated) forecaster. We finally compare our results to
some existing calibrated forecasters in the literature.

The main result of this section is
providing rates of convergence for a calibrated forecaster in~(\ref{eq:ratescalibr}). To the best of our knowledge, this is the first rates results for calibration for an alphabet of size $A$ larger than 2.
For $A = 2$, (sub)optimal rates follow from the procedure of Foster and Vohra~\cite{FoVo98} as recalled in
Section~\ref{sec:FoVo}.

\subsection{Rates of convergence.}
\label{sec:rates}

Approachability theory provides uniform convergence rates of sequence of empirical
payoff vectors to the target set, see Cesa-Bianchi and Lugosi~\cite[Exercise 7.23]{CeLu06}. Formally, denoting
by $\norm_2$ the Euclidian $\ell^2$--norm in $\R^{A N_\e}$, it follows in our context that
there exists some absolute constant $\gamma$
(independent of $A$ and $N_\e$) such that for all strategies of Nature
and for all $T$, with probability $1-\delta$,
\[
\norm[ \overline{m}_T - \Pi_C \bigl( \overline{m}_T \bigr) ]_2 \leq \gamma \sqrt{\frac{\ln (1/\delta)}{T}}~.
\]
Here, it is crucial to state the convergence rates based on the Euclidian norm because of an
underlying martingale convergence argument in Hilbert spaces proved by Chen and White~\cite{ChWh96}.
The reason why the convergence rate here is independent of $A$ and $N_\e$
is that the payoff vectors $m(k,a)$ all have an Euclidian norm bounded by an absolute constant, e.g., 2;
this happens because most of their components are 0.

We now apply this result. However, we underline that the set $C$ can be defined by a different calibration norm $\norm$;
below, we will define it based on the $\ell^1$--norm, for instance. But the stated uniform
convergence rate can be used since, via a triangle inequality and an application of the
Cauchy-Schwarz inequality,
\[
\norm[ \overline{m}_T ]_1 \leq \norm[ \Pi_C \bigl( \overline{m}_T \bigr) ]_1 +
\norm[ \overline{m}_T - \Pi_C \bigl( \overline{m}_T \bigr) ]_1 \leq
\e + \sqrt{A N_\e} \, \norm[ \overline{m}_T - \Pi_C \bigl( \overline{m}_T \bigr) ]_2~.
\]

$N_\e$ is of the order of $\e^{-(A-1)}$; we let $\gamma'$ be an absolute constant such
that $N_\e \leq \gamma' \, \e^{-(A-1)}$ for all $\e \leq 1$ (say).
We therefore have proved that given $0 < \e \leq 1$, the forecaster defined in the previous section
is such that
for all strategies of Nature
and for all $T$, with probability $1-\delta$,
\[
\norm[ \overline{m}_T ]_1 =
\sum_{k=1}^{N_\e} \norm[\frac{1}{T} \sum_{t=1}^T \ind_{ \{ K_t = k \} }
\bigl( \bp_k - \delta_{a_t} \bigr)]_1 \leq \e + \gamma \gamma' \sqrt{A} \,
\sqrt{\frac{\ln (1/\delta)}{\e^{(A-1)} \, T}} \stackrel{\mbox{\scriptsize def}}{=} U_{\e,T,\delta}~.
\]
This high-probability bound is to be used below as the key ingredient to construct a calibrated
forecaster, i.e., a forecaster satisfying~(\ref{def:cal}).
Combining the Borel-Cantelli Lemma  with the bound above shows that the less ambitious goal~(\ref{def:epscal})
can be achieved.

\subsection{Construction of a calibrated forecaster.}
\label{sec:cstr}

We use a standard approach which is commonly known as the ``doubling trick," see, e.g., Cesa-Bianchi and 
Lugosi~\cite{CeLu06}. It consists of defining
a meta-forecaster that proceeds in regimes; regime $r$ (where $r \geq 1$)
lasts $T_r$ rounds and resorts for the forecasts
to an $\e_r$--calibrated forecaster, for some $\e_r > 0$ to be defined by the analysis.
We now show that for appropriate values of the $T_r$ and $\e_r$, the resulting meta-forecaster is calibrated
in the sense of (\ref{def:cal}), and even uniformly calibrated in the following sense,
where $\cB$ denotes the Borel sigma-algebra of $\cP$:
\begin{equation}
\label{def:unifcal}
\lim_{T \to +\infty} \ \ \sup_{B \in \cB} \ \norm[\frac{1}{T} \sum_{t=1}^T \ind_{ \{ P_t \in B \} }
\bigl( P_t - \delta_{a_t} \bigr)] \,\, = 0 \qquad \quad \mbox{a.s.}
\end{equation}
Of course, uniform calibration (\ref{def:unifcal}) implies calibration (\ref{def:cal}) via the choices
for $B$ given by $\e$--balls around probability distributions $\bp$.

For concreteness, we focus below on the $\ell^1$--calibration score.
\bigskip

Regimes are indexed by $r = 1,2,\ldots$ and the index of the regime corresponding to round $T$
is referred to as $R_T$. The set of the rounds within regime $r \leq R_T -1$
is called $\cT_r$; rounds in regime $R_T$ with index less than $T$ are
gathered in the set $\cT_{R_T}$ (we commit here an abuse of notations).
We denote by $\bp_{k,r}$, where $k \in \{ 1,\ldots,N_{\e_r} \}$, the finite $\e_r$--grid
considered in the $r$--th regime. By the triangle inequality satisfied by $\norm$, we first decompose
the quantity of interest according to the regimes and to the played points of the grids,
\[
\norm[\sum_{t=1}^T \ind_{ \{ P_t \in B \} }
\bigl( P_t - \delta_{a_t} \bigr)]_1 \leq \sum_{r=1}^{R_T}
\sum_{k = 1}^{N_{\e_r}} \ind_{ \{ \bp_{k,r} \in B \} } \norm[ \sum_{t \in \cT_r}
\ind_{ \{ K_t = k \} } \bigl( \bp_{k,r} - \delta_{a_t} \bigr)]_1~.
\]
We now substitue the uniform bound obtained in the previous section and get that with
probability $1 - (\delta_{1,T} + \ldots + \delta_{R_T,T}) \geq 1 - 1/T^2$,
\[
\sup_{B \in \cB} \norm[\frac{1}{T}\sum_{t=1}^T \ind_{ \{ P_t \in B \} }
\bigl( P_t - \delta_{a_t} \bigr)]_1 \leq
\frac{1}{T} \sum_{r=1}^{R_T} T_r \, U_{\e_r,T_r,\delta_{r,T}}~,
\]
where we defined $\delta_{r,T} = 1/(2^r T^2)$.

An application of the Borel-Cantelli Lemma and Cesaro's Lemma shows that for
suitable choices of a sequence $\e_r$ decreasing towards 0 and an increasing sequence
$T_r$ such that $\e_r^{A-1} \, T_r$ tends to infinity fast enough, one then gets
the desired convergence (\ref{def:unifcal}).
For instance, if $T_r = 2^r$,
and $\e_r$ is chosen such that
\[
\e_r \qquad \mbox{and} \qquad \sqrt{\frac{1}{\e_r^{\, (A-1)} \, T_r}}
\]
are of the same order of magnitude, e.g., $\e_r = 2^{-r/(A+1)}$, then
\begin{equation}
\label{eq:ratescalibr}
\limsup_{T \to \infty} \ \  \frac{T^{1/(A+1)}}{ \sqrt{\ln T}} \,\,
\sup_{B \in \cB} \norm[\frac{1}{T}\sum_{t=1}^T \ind_{ \{ P_t \in B \} }
\bigl( P_t - \delta_{a_t} \bigr)]_1 \ \leq \Gamma_A \qquad \quad \mbox{a.s.}~,
\end{equation}
where the constant $\Gamma_A$ depends only on $A$. As indicated above, to the best of our knowledge, this is the first
rates results for calibration for an alphabet of size $A$ larger than 2.

\subsection{Comparison to previous forecasters.}
\label{sec:Brier}

\subsubsection{$\ell^1$--calibration score.}

Foster~\cite{Fo99} first considered the $\ell^1$--calibration score in the context of the prediction
of binary outcomes only, i.e., when $A = 2$. The $\e$--calibrated forecaster he explicitly
exhibited has a computational complexity of the order of $1/\e$. He did not work out the convergence
rates but since his procedure is mostly a clever twist on our general procedure, they should be
similar to the ones we proved in Section~\ref{sec:rates}.

\subsubsection{Brier score.}
\label{sec:FoVo}

What follows is extracted from Foster and Vohra~\cite{FoVo98}; see also Cesa-Bianchi and Lugosi~\cite[Section 4.5]{CeLu06}.

Given an $\e$--grid over the simplex $\cP$, we define, for all $k \in \{ 1,\ldots,N_\e \}$,
the empirical distribution of the outcomes chosen by Nature at those rounds $t$ when
the forecaster used $\bp_k$,
\begin{numcases}{\rho_T(k) = }
\nonumber
\bp_k & if $\sum_{t=1}^T \ind_{ \{ K_t = k \} } = 0$, \vspace{.15cm} \\
\nonumber
\sum_{t=1}^T \ind_{ \{ K_t = k \} } \, \frac{1}{\sum_{t=1}^T \ind_{ \{ K_s = k \} }} \, \delta_{a_t}
& if $\sum_{t=1}^T \ind_{ \{ K_t = k \} } > 0$.
\end{numcases}

The classical Brier score can be shown in our setup to be equal to the following criterion:
\[
\sum_{k=1}^{N_\e} \norm[\rho_T(k) - \bp_k]_2^2 \, \left( \frac{1}{T} \sum_{t=1}^T \ind_{ \{ K_t = k \} } \right)~.
\]
Since for two probability distributions $\bp$ and $\bq$ of $\cP$, one always has
\[
\norm[\bp - \bq]_2^2 \leq 2 \norm[\bp - \bq]_1~,
\]
the Brier score can be seen to be upper bounded by twice the $\ell^1$--calibration score; it is thus a weaker
criterion.

\bigskip
Cesa-Bianchi and Lugosi~\cite[Section~4.5]{CeLu06} shows however that forecasters with Brier scores asymptotically smaller
than $\e$ can be the keystones to construct calibrated forecasters, in a way similar to the construction exhibited
in Section~\ref{sec:cstr}.

In the case $A = 2$, these forecasters essentially bound the Brier score, with probability at least $1-\delta$, by
a term that is of the order of
\[
\e + \frac{1}{\e} \sqrt{\frac{\ln(1/\e) + \ln(1/\delta)}{T}}~,
\]
which is worse than the rate we could exhibit in Section~\ref{sec:rates} for the
$\ell^1$--calibration score.

In addition, the computational complexity of the underlying procedure (based on the minimization of
internal regret) is of the order of $1/\e^2$ per stage and
thus is similar to the complexity $1/\e^{A+1} = 1/\e^2$ we derived in Section~\ref{sec:cplx}
for our new procedure.

\bigskip
The general case of $A \geq 3$ is briefly mentioned in Cesa-Bianchi and Lugosi~\cite[Section~4.5]{CeLu06} indicating
that the case of $A=2$ can be extended to $A \geq 3$ without further details.
As far as we can say,
the computational complexity of such an extension per step would be of the order of
$1/\e^{2(A-1)}$ versus $1/\e^{(A+1)}$ for the
approachability-based procedure we suggested above.
The convergence rates, for a straightforward extension, seem to be quite slow.
However, based on a draft of the present article, Perchet~\cite{Per10} recently
proposed a more efficient extension of the procedure of Foster and Vohra~\cite{FoVo98} and obtained the
same rates of convergence as in~(\ref{eq:ratescalibr}); he however did not work
out the complexity of his procedure, which seems to be similar to the one of our construction.

\section{Acknowledgments.}
Shie Mannor was partially supported by the ISF under contract 890015 and a Horev Fellowship.
Gilles Stoltz was partially supported by the French ``Agence Nationale pour la Recherche''
under grant JCJC06-137444 ``From applications to theory in learning and adaptive statistics''
and by the PASCAL Network of Excellence under EC grant {no.} 506778.

{\small
\bibliographystyle{amsplain}
\bibliography{Mannor-Stoltz--Geometric-Calibation--Final}
}

\end{document}